\documentclass[twocolumn,10pt]{article}
\usepackage{amsmath,amssymb,amsfonts,mathrsfs}
\usepackage[round,authoryear]{natbib}
\let\cite\citep
\usepackage{graphicx}
\usepackage{dblfloatfix}
\usepackage{titling}
\usepackage{hyperref}
\hypersetup{colorlinks=true,citecolor=cyan}

\allowdisplaybreaks[0]

\newtheorem{theorem}{Theorem}
\newtheorem{definition}{Definition}
\newtheorem{lemma}{Lemma}

\newtheorem{assumption}{Assumption}
\def\rr{{\mathbb{R}}}
\def\al{{\alpha}}
\def\BibTeX{{\rm B\kern-.05em{\sc i\kern-.025em b}\kern-.08em
        T\kern-.1667em\lower.7ex\hbox{E}\kern-.125emX}}
\newenvironment{proof}{\noindent\textit{Proof.}}{\hfill$\square$\par}

\title{On Global Regulatability of Robot Manipulators\\[0.2em]
by Classical PID\thanks{This work was supported in part by National Key R\&D Program of China under Grant 2024YFA1013104, and by the National Natural Science Foundation of China under Grants U22B6001, 12288201, and 62303451. (Corresponding author: Lei Guo)}}
\author{Cheng Zhao (\texttt{zhaocheng@amss.ac.cn}) \and
Jingru Zhu (\texttt{zhujingru@amss.ac.cn}) \and
Lei Guo (\texttt{lguo@amss.ac.cn})\\[0.5em]
\small State Key Laboratory of Mathematical Sciences, AMSS, Chinese Academy of Sciences, Beijing 100190, China\\
\small School of Mathematical Sciences, University of Chinese Academy of Sciences, Beijing 100049, China}
\date{}

\begin{document}
\maketitle

\begin{abstract}
A long-standing open problem in robot manipulator control is whether global
regulation can be achieved by classical PID control. This paper provides an
answer to this question for classical PID controllers with triple parameters \((k_p,k_i,k_d)\in\mathbb R^3\).
We find and prove that for one-degree-of-freedom manipulators, the classical PID control guarantees global stability and asymptotic regulation under standard structural assumptions, and further derive explicit quantitative design conditions for the PID gains. 
However, for multi-degree-of-freedom cases, we can construct a robot manipulator satisfying the same structural assumptions for which no choice of PID gains $(k_p,k_i,k_d)$ can achieve global asymptotic regulation.
These results provide a fundamental understanding of the abovementioned open problem, revealing both the fundamental capability and intrinsic limitation  of the classical PID control for robot manipulator dynamics.
\end{abstract}
\medskip

\section{Introduction}\label{sec:introduction}

Feedback control is a fundamental paradigm for steering dynamical systems toward desired objectives, enabling stability, regulation, disturbance rejection, and performance improvement across a wide range of applications. Among various feedback strategies, proportional-integral-derivative (PID) control remains one of the most widely adopted control methodologies in practice due to its simple structure, low implementation cost, and remarkable engineering effectiveness, continuing to serve as a fundamental design paradigm in industrial control systems despite the rapid development of modern control theory \cite{astrom2001,Samad2017A,mao2024}.

The remarkable success of PID control naturally raises a fundamental question: why can such a simple linear feedback structure achieve reliable performance for nonlinear dynamical systems, and how can its parameters be systematically designed to guarantee closed-loop stability? 

In recent years, a systematic theoretical framework for PID control of uncertain nonlinear systems has been developed; see for example, \cite{2017zc,zhang2019,towards,guo2020,lin2022,zhu2024,li2024}.  Within this framework, rigorous PID stability results have been established for several classes of second-order nonlinear systems described by Newtonian dynamics, including necessary and sufficient conditions for global stabilization and explicit characterizations of admissible PID gain regions. These results provide a theoretical foundation for PID design beyond empirical tuning and demonstrate that PID parameters can be determined systematically rather than through trial-and-error procedures.

An important feature of this framework is its broad applicability: the analysis does not require a precise model of the system dynamics, but relies only on limited prior information, such as Lipschitz bounds and growth-rate conditions
on uncertain nonlinearities \cite{dilemma,2017zc,towards}. However, this generality may come at the cost of leaving intrinsic physical structures underexploited. Examples of such structures include conservative vector fields, energy conservation or dissipation properties, and Lagrangian structures
\cite{Brogliato,Ortega2002}. Structure-aware analysis has proved useful in related PID-type control problems, including nonlinear anti-windup design for input-saturated Euler--Lagrange systems and differential-inclusion analysis of PID-controlled systems with Coulomb friction \cite{lu1,lu2}. These developments
suggest that exploiting system structure may lead to stability guarantees that are more refined, or even stronger, than those obtained under generic uncertainty assumptions.
A recent study indicates that this is indeed possible. For example, for a class of second-order multi-input multi-output nonlinear systems with conservative position-dependent vector fields, it has been shown that classical PID controllers with scalar gains can guarantee global stabilization despite the presence of highly nonlinear dynamics \cite{xty-zc}.

Among physical systems possessing such intrinsic structures, robot manipulators represent a fundamental class of nonlinear mechanical systems characterized by Euler--Lagrange dynamics, with a positive definite inertia matrix, Coriolis and centrifugal effects, and gravitational forces \cite{kelly2005,Brogliato}. Owing to their widespread applications in industrial automation, medical robotics, and autonomous manipulation systems, robot manipulators have served as canonical benchmark systems for the development and validation of nonlinear control theories for several decades.

In the field of robot control, a notable result by Takegaki and Arimoto showed that proportional-derivative (PD) control can globally stabilize frictionless manipulators in the absence of external disturbances \cite{new1981}. However, pure PD feedback does not in general provide global regulation for arbitrary desired configurations; without gravity compensation, the desired configuration must itself be an equilibrium of the manipulator. Related work further showed that PD control combined with gravity compensation provides a simple global regulator when the gravitational torque is accurately known \cite{Arimoto}. This model-based compensation, however, requires precise model parameters, such as the payload mass, whereas pure PD control generally cannot reject persistent disturbances or eliminate steady-state errors when the desired configuration is not a natural equilibrium.

From an engineering perspective, adding an integral term appears to be a natural solution since the integral action is specifically designed to reject constant disturbances and remove steady-state errors. This motivates the use of classical PID controllers, which are widely implemented in industrial robot systems and do not require a complete dynamic model in the control law. Nevertheless, from a theoretical viewpoint, the introduction of integral action makes the stability analysis substantially more challenging. Unlike PD control whose global stability has been well established, whether the classical linear PID control can globally asymptotically stabilize robot manipulators has remained an open problem for several decades \cite{kelly1998,Rocco1996,ifac2008saturated}.

Kelly emphasized this gap, noting that a complete global stability proof for PID-controlled robot manipulators was still unavailable despite its widespread practical use \cite{kelly1998}. Subsequent studies pursued different directions, including semiglobal analysis of the classical PID controller and the development of modified PID-type structures. Semiglobal stability results were obtained for PID-type regulators with additional structural modifications \cite{Ortega1995,Alvarez-Ramirez}, while the classical PID controller itself was shown to achieve semiglobal regulation under suitable conditions \cite{auto-semiglobal}. Global regulation was also achieved in related works by introducing nonlinear integral actions, saturation mechanisms, or other structural modifications \cite{kelly1998,Sun2009,Su2010}. These results demonstrate that global regulation can be recovered by changing the controller structure, but they do not answer the fundamental question:
can the classical PID controller with its standard linear structure and scalar gains globally regulate robot manipulators?

This paper provides a dimension-dependent resolution of this question for the classical PID control with triple parameters $(k_p,k_i,k_d)$.
Specifically, we consider robot manipulators governed by standard Euler--Lagrange dynamics and controlled by classical PID feedback with gains $(k_p,k_i,k_d)\in\mathbb{R}^3$. For one-degree-of-freedom manipulators, we prove that classical PID control guarantees global stability and asymptotic regulation under standard structural assumptions, and derive explicit quantitative design conditions for the PID gains. The proof exploits the intrinsic mechanical structure of the system by constructing an energy-based Lyapunov function, consisting of a quadratic form of the augmented error states and the gravitational potential energy, which reveals how the PID gains interact with the underlying Euler--Lagrange dynamics.

In contrast, for manipulators with $n\geq2$ degrees of freedom, we construct counterexamples satisfying all the standard structural assumptions of rigid robot manipulators. For these systems, we show that for any choice of scalar PID gains $(k_p,k_i,k_d)\in\mathbb{R}^3$, the corresponding closed-loop solutions remain globally defined for all initial conditions, yet the regulation objective fails, with the tracking error and its derivative not converging to zero. This establishes that the failure of global regulation in higher-dimensional robot manipulators is not due to the loss of global existence of the closed-loop solutions, but rather reflects a fundamental limitation of the classical PID control with triple parameters $(k_p,k_i,k_d)$. 

In summary, this work establishes a dimension-dependent characterization of the capability and limitation of the classical PID control with triple parameters $(k_p,k_i,k_d)$ for the global regulation of robot manipulators.

\subsection*{Notation}
Throughout the paper, for a vector $x\in\mathbb R^n$, $|x|=(x^\top x)^{1/2}$ denotes its Euclidean norm. For a matrix $A\in\mathbb R^{m\times n}$, $\|A\|$ denotes the induced Euclidean norm, $\|A\|=\sup_{|x|=1}|Ax|=\sqrt{\lambda_{\max}(A^\top A)}$.
For symmetric matrices $A,B\in\mathbb R^{n\times n}$, $A\le B$ means that $B-A$ is positive semidefinite, equivalently $x^\top Ax\le x^\top Bx$ for all $x\in\mathbb R^n$. The identity matrix of dimension $n$ is denoted by $I_n$, or simply by $I$ when the dimension is clear.

\section{Problem Formulation}\label{sec:problem}
We consider an $n$-degree-of-freedom (DOF)  rigid robot manipulator governed by the Euler--Lagrange dynamics
\begin{align}\label{EL}
    M(q)\ddot q+C(q,\dot q)\dot q+g(q)=\tau,
    \quad q\in\mathbb{R}^n,\ \tau\in\mathbb{R}^n,
\end{align}
where $q,\dot q,\ddot q\in\mathbb{R}^n$ denote the generalized
coordinate, velocity, and acceleration vectors, respectively.
$M(q)$ is the inertia matrix, $C(q,\dot q)\dot q$ represents the
Coriolis and centrifugal term, $g(q)$ is the gravitational force, and $\tau\in\mathbb{R}^n$ is the control input.

The following assumption summarizes the structural properties of rigid robot manipulators considered in this paper, including the standard Euler--Lagrange properties and an additional symmetry condition on the Coriolis mapping \cite{kelly2005,Murray1994,Su2010}.

\begin{assumption}\label{ass:robot}
There exist positive constants $m_0$, $m_1$, $L_c$, $L_g$ such that for all $q,v,w\in\mathbb{R}^n$:
\begin{align*}
    &(p1)\quad m_0I\leq M(q)\leq m_1I,\\
    &(p2)\quad \dot M(q)-2C(q,\dot q)\ \text{is skew-symmetric},\\
    &(p3)\quad \|C(q,v)\|\leq L_c|v|,\\
    &(p4)\quad C(q,v)w=C(q,w)v,\\
    &(p5)\quad g(q)=\nabla U(q),
      U\in C^2(\mathbb{R}^n;\mathbb{R}),\|\nabla^2 U(q)\|\leq L_g,
\end{align*}
where $U(q)$ is the gravitational potential energy and the boundedness of its Hessian $\nabla^2 U(q)$ implies that the gravity vector field is globally Lipschitz continuous.
\end{assumption}

Given an arbitrary constant desired configuration $q_d\in\mathbb{R}^n$, the control objective is to drive the manipulator to $q_d$ from arbitrary initial conditions. Define the position error as
\[
e(t)=q_d-q(t).
\]
The regulation objective is
\[
e(t)\to0,\qquad \dot e(t)\to0,\qquad \text{as }t\to\infty .
\]
In this paper, we focus on the classical PID controller
\begin{align}\label{pid}
    \begin{split}
        \tau(t)=&k_i \xi(t)+k_p e(t)+k_d \dot e(t),\\ \dot \xi(t)=&e(t),
    \end{split}
\end{align}
where $(k_p,k_i,k_d)\in\mathbb{R}^3$ are the scalar PID gains, and $\xi(t)\in\mathbb{R}^n$ is the integral state.  Thus, we restrict attention to the scalar-gain PID architecture, where the same
triple \((k_p,k_i,k_d)\) is applied to all degrees of freedom; the case of
matrix-valued PID gains is  left for future investigation.

\begin{definition}\label{def:admissible} We say that the robot manipulator system \eqref{EL} can be \emph{globally regulated} by the PID controller \eqref{pid}, if there exists a gain vector 
\[ K=(k_p,k_i,k_d)\in\mathbb{R}^3 \] 
such that, for every initial condition $(q(0),\dot q(0),\xi(0))\in\mathbb{R}^{n}\times\mathbb{R}^{n}\times \mathbb{R}^{n}$ and every desired configuration $q_d\in\mathbb{R}^n$, the corresponding closed-loop system admits a solution defined for all $t\geq0$ and satisfies 
\[ e(t)\to0,\qquad \dot e(t)\to0, \qquad \text{as }t\to\infty . \] 
Such a gain vector $K$ is called an \emph{admissible PID gain}. \end{definition}

The global regulation of robot manipulators under classical PID control has remained a challenging problem in nonlinear control theory. Robot manipulators are widely used in industrial automation, surgical robotics, and other engineering applications, while classical PID control remains one of the most widely adopted feedback strategies by far. Therefore, establishing whether classical PID control can globally regulate rigid robot manipulators satisfying standard structural assumptions is a central question. Surprisingly, despite its simplicity and extensive practical success, this question has remained unresolved for decades.


In the next section, we  show that the answer is intrinsically dimension-dependent: for one-degree-of-freedom manipulators, these structural properties are sufficient and lead to
explicit PID gain conditions; however, for manipulators with two or more
degrees of freedom, they are not sufficient in general. Thus, the global PID regulation problem admits a \emph{dimension-dependent answer}.

\section{The Main Results}
\subsection{Existence of admissible PID gains in the one-DOF case}

We first consider the one-degree-of-freedom case, i.e., $n=1$. 
To characterize the PID gains that achieve global regulation, we introduce the
following admissible PID gain set:
\begin{align}
\label{eq:gain_condition}
\scalebox{0.9}{$
\mathcal K_1
\triangleq
\left\{
\begin{pmatrix}
	k_p\\k_i\\k_d
\end{pmatrix}
\,\left| 
\begin{array}{l}
k_i>0,\quad k_d>0,\\[3pt]
k_p>L_g+\max\left\{
\dfrac{k_d^2}{2m_0},~
4m_1\dfrac{k_i}{k_d}
+\dfrac{L_g^2}{m_0}\dfrac{k_d}{k_i}
\right\}
\end{array}\right.
\right\}$}
\end{align}
The gain condition \eqref{eq:gain_condition} gives a simple and robust way to select the PID gains. Importantly, it depends only on the uniform inertia bounds $m_0,m_1$ and the global slope bound $L_g$ of the gravity term, rather than on the exact expressions of $M(\cdot)$, $C(\cdot,\cdot)$, or $g(\cdot)$.

\begin{theorem}
\label{thm}
Consider the one-degree-of-freedom robot manipulator governed by \eqref{EL}.
If Assumption \ref{ass:robot} is satisfied,
then for any desired setpoint $q_d\in\mathbb R$, and for any PID controller \eqref{pid} with gains $(k_p,k_i,k_d)\in\mathcal K_1$, 
the closed-loop system achieves global regulation with exponential rate: for any initial condition $\xi(0),q(0),\dot q(0)\in \mathbb{R}$,
\[
|e(t)|+|\dot e(t)|\to 0 \text{ exponentially as } t\to\infty .
\]
\end{theorem}
\textbf{Remark 1 }
Beyond providing a sufficient stability condition, Theorem~\ref{thm} also
reveals an important feature of PID gain design for one-degree-of-freedom
robot manipulators. In particular, the integral gain $k_i$ and the derivative
gain $k_d$ can be selected as arbitrary positive constants. Once $k_i>0$
and $k_d>0$ are fixed, the proportional gain $k_p$ can always be chosen
 sufficiently large such that \eqref{eq:gain_condition} is satisfied. 
 
Therefore, the derivative gain \(k_d\) can be chosen arbitrarily small,  with no
strictly positive lower bound on \(k_d\) required for global regulation. This
property follows from the energy-balance structure  of Euler--Lagrange systems, where the
derivative feedback $k_d\dot e(t)$ provides natural dissipation in the closed-loop dynamics.

Such a mild requirement differs from general second-order
nonlinear systems without such structure, where explicit lower bounds on
\(k_d\) may be needed to dominate nonlinear terms, such as \(k_d>L_2\) in
\cite{2017zc}.

\textbf{Remark 2 }
The preceding discussion shows that the integral and derivative gains
\(k_i>0\) and \(k_d>0\) can be selected arbitrarily small, provided that the
proportional gain \(k_p\) is chosen sufficiently large. This naturally raises
the question of how small \(k_p\) can be while still satisfying the proposed
gain condition. The following argument provides an explicit sufficient
threshold.

Suppose that
\[
k_p>L_g\left(1+4\sqrt{\frac{m_1}{m_0}}\right).
\]
Then, for any sufficiently small \(\varepsilon>0\), choose
\[
k_d=\varepsilon,\qquad
k_i=\frac{L_g}{2\sqrt{m_0m_1}}\varepsilon .
\]
Under this choice,
\[
4m_1\frac{k_i}{k_d}
+\frac{L_g^2}{m_0}\frac{k_d}{k_i}
=
4L_g\sqrt{\frac{m_1}{m_0}},
\]
while
\[
\frac{k_d^2}{2m_0}
=
\frac{\varepsilon^2}{2m_0}
\rightarrow0
\qquad\text{as }\varepsilon\rightarrow0^+.
\]
Therefore, for sufficiently small \(\varepsilon\), the gain condition
\eqref{eq:gain_condition} is satisfied.

The above threshold, although conservative, has a clear interpretation. Since
the gravity term satisfies
\[
|g(q)-g(q_d)|\leq L_g|q-q_d|,
\]
the proportional feedback must dominate the worst-case variation of the
gravity force. Hence, the lower bound on \(k_p\) represents a gravity-dominance
requirement, while the dependence on \(m_0\) and \(m_1\) reflects the influence
of the inertia bounds arising from the Lyapunov analysis.

\subsection{Impossibility of global PID regulation for multi-DOF manipulators}

In this subsection, we consider the multi-degree-of-freedom case, i.e.,
$n\geq 2$. The following theorem shows that, unlike the one-degree-of-freedom
case, the same class of structural assumptions is insufficient to guarantee
global regulation by scalar PID gains.

\begin{theorem}
\label{thm:higher_dim_negative}
For any $n\geq 2$, there exists an $n$-degree-of-freedom
robot manipulator \eqref{EL} satisfying
Assumption \ref{ass:robot} that is not globally regulable by any PID controller \eqref{pid}.
More precisely, for this manipulator, for any gain triple
$
    (k_p,k_i,k_d)\in\mathbb R^3,
$
there exist a desired setpoint $q_d\in\mathbb R^n$ and an initial condition
$
    (q(0),\dot q(0),\xi(0))
$,
such that the corresponding closed-loop solution is complete, but fails to
achieve regulation, namely
\[
    \limsup_{t\to\infty}
    \bigl(|e(t)|+|\dot e(t)|\bigr)
    \neq 0 .
\]
\end{theorem}
\textbf{Remark 3 }
Theorem~\ref{thm:higher_dim_negative} reveals a fundamental limitation of
classical PID control for higher-dimensional robot manipulators. It is worth distinguishing our result from the finite-escape-based impossibility results for general nonlinear systems. For example, Proposition~1 in \cite{zc-capability} shows that under certain superlinear growth conditions, the closed-loop system may exhibit a finite escape time, which prevents global stabilization due to the loss of solution existence.

In contrast, in the counterexample constructed in this paper, the closed-loop
trajectory used to disprove global regulation is complete for all
$t\geq0$, while the regulation objective still fails because the trajectory
does not converge to the desired setpoint. Hence, the impossibility result
here arises from the lack of global convergence rather than from finite escape
or loss of well-posedness of the closed-loop dynamics.

Finally, we point out that Theorem~2 is established for the classical PID
architecture with scalar gains. Determining whether appropriately designed
matrix-valued PID gains can achieve global regulation for higher-dimensional
robot manipulators remains an open problem and a promising direction for
future research.

\section{Proof of the Main Results}\label{sec:proof-1d}

\subsection{Proof of Theorem 1}
In the one-degree-of-freedom case, the inertia and Coriolis terms are scalar-valued. We therefore denote them by $m(q)$ and $c(q,\dot q)$, respectively, both to simplify the notation and to distinguish the subsequent one-DOF analysis from the $n$-DOF case. 

In addition, Assumption \ref{ass:robot}
implies
\[
    m_0\le m(q)\le m_1,
    \quad
    |g(q)-g(q_d)|\le L_g|q-q_d|,
\]
and the skew-symmetry identity $(p2)$ reduces, along trajectories, to
\[
    \frac{\mathrm{d}}{\mathrm{d}t}m(q(t))=2c(q(t),\dot q(t)).
\]
Now, suppose that the PID gains $(k_p,k_i,k_d)\in\mathcal K_1$. We divide the proof into several steps.

\noindent\emph{Step 1:  Reformulation of the closed-loop system via augmented coordinates.}

Define the state variables
\begin{equation}\label{eq:zdef}
\begin{aligned}
    z_0 &:= \xi - k_i^{-1} g(q_d), \\
    z_1 &:= q_d - q, \\
    z_2 &:= -\sqrt{m(q)}\,\dot q.
\end{aligned}
\end{equation}
Also define
\begin{equation}\label{eq:alphabeta}
    \beta(q) := \frac{1}{\sqrt{m(q)}},
    \quad
    \alpha(q) :=
    \begin{cases}
        \dfrac{g(q) - g(q_d)}{q - q_d}, & q \neq q_d, \\[8pt]
        g'(q_d), & q = q_d.
    \end{cases}
\end{equation}
Then $\dot q = -\beta(q)z_2$ and $g(q)-g(q_d)=-\al(q)z_1$. By the scalar consequences of Assumption~\ref{ass:robot}, we have $|\al(q)|\leq L_g$ for all $q\in\rr$.
Let us denote $\beta_0=\frac{1}{\sqrt{m_1}}$ and $\beta_1=\frac{1}{\sqrt{m_0}}$. 
By the scalar consequences of Assumption~\ref{ass:robot}, we have $0<\beta_0\leq \beta(q)\leq\beta_1$, for all $q\in\rr$. Hence, 
\begin{align}\label{eq:alpha_beta_bounds}
    |\al(q)|\leq L_g,\qquad \beta_0 \leq \beta(q) \leq \beta_1.
\end{align}
From \eqref{eq:zdef}, $\dot z_0 = q_d - q = z_1.$ Since $z_1 = q_d - q$, one has $\dot z_1 = -\dot q.$ 
Moreover, $z_2 = -\sqrt{m(q)}\,\dot q = -\frac{\dot q}{\beta(q)},$ hence $\dot z_1 = \beta(q) z_2.$
Next, differentiating $z_2 = -\sqrt{m(q)}\,\dot q$ gives
\begin{equation}\label{eq:z2dot_first}
\begin{aligned}
    \dot z_2
    &=
    -\frac{\dot m(q)}{2\sqrt{m(q)}}\dot q
    - \sqrt{m(q)}\,\ddot q \\
    &=
    -\frac{1}{2}\beta(q) \dot m(q) \dot q
    + \frac{c(q,\dot q)\dot q+g(q)-\tau}{\sqrt{m(q)}}.
\end{aligned}
\end{equation}
From \eqref{eq:zdef}, we have 
\begin{align}\label{eq:tau_z}
    \tau = k_i z_0  + k_p z_1 + k_d\beta(q) z_2+ g(q_d).
\end{align}
Substituting $\dot m(q)=2c(q,\dot q)$ and \eqref{eq:tau_z}  into \eqref{eq:z2dot_first} yields 
\begin{align*}\label{eq:z2dot}
 \dot z_2=-\beta(q)\left(
        k_i z_0
        + \left(k_p + \alpha(q)\right) z_1
        + k_d \beta(q) z_2
    \right).
\end{align*}
Hence the closed-loop system is $\dot z = \beta(q) A(q) z$, where $z = \begin{bmatrix}
    z_0&z_1&z_2
\end{bmatrix}^\top$ and
\begin{equation}\label{eq:A_def}
    A(q)
    =
    \begin{bmatrix}
        0 & 1/\beta(q) & 0 \\
        0 & 0 & 1 \\
        -k_i & -(k_p+\alpha(q)) & -k_d\beta(q)
    \end{bmatrix}.
\end{equation}
\noindent\emph{Step 2: Construction of an energy-based Lyapunov function.}
Define
\begin{equation}\label{eq:U_def}
    \Psi(q) := \int_{q_d}^{q} \bigl(g(s) - g(q_d)\bigr)\,\mathrm{d}s.
\end{equation}
Here, \(\Psi(q)\) represents the shifted gravitational potential energy
relative to the desired configuration. The subtraction of \(g(q_d)\)
ensures that the potential term vanishes in gradient at \(q=q_d\), namely,
\(\nabla\Psi(q_d)=0\).

We consider the Lyapunov candidate
\begin{equation}\label{eq:V_comp}
V(z):=z^\top P z
    + 2p_{33} \Psi(q)
\end{equation}
\begin{figure*}[!b]
\hrule
\vspace{0.5em}
\begin{align}
    R(q)&=\begin{bmatrix}
        2k_ip_{13}&~~~(\al(q)+k_p)p_{13}+k_ip_{23}-\frac{1}{\beta(q)}&k_d\beta(q) p_{13}+k_ip_{33}-p_{12}\\
        *&2(\al(q)+k_p)p_{23}-\frac{2p_{12}}{\beta(q)}&~~~k_d\beta(q) p_{23} + k_pp_{33}-p_{22}-\frac{p_{13}}{\beta(q)}\\
        *&*&2k_d \beta(q) p_{33}-2p_{23}
    \end{bmatrix}\label{eq:Rq_full}\\[0.4em]
    &\qquad R=\begin{bmatrix}
        2k_ip_{13}&~~(\al(q)+k_p)p_{13}+k_ip_{23}-\frac{1}{\beta(q)}&k_d\beta(q) p_{13}\\
        *&2(\al(q)+k_p)p_{23}-\frac{2p_{12}}{\beta(q)}&~~k_d\beta(q) p_{23} -\frac{p_{13}}{\beta(q)}\\
        *&*&~~2k_d \beta(q) p_{33}-2p_{23}
    \end{bmatrix}
    \label{dotV}
\end{align}
\vspace{0.3em}
\begin{equation}\label{eq:R_lemma}
    R(\al,\beta)=\begin{bmatrix}
        2k_ip_{13}&~~~(\al+k_p)p_{13}+k_ip_{23}-\frac{1}{\beta}&k_d\beta p_{13}\\
        *&2(\al+k_p)p_{23}-\frac{2p_{12}}{\beta}&~k_d\beta p_{23} -\frac{p_{13}}{\beta}\\
        *&*&~~2k_d \beta p_{33}-2p_{23}
    \end{bmatrix}
\end{equation}
\vspace{-0.4em}
\end{figure*}
where $P$ is given by
\begin{equation*}\label{eq:P_def}
    P =
    \begin{bmatrix}
        1&p_{12}&p_{13}\\
        p_{12}&p_{22}&p_{23}\\
        p_{13}&p_{23}&p_{33}
    \end{bmatrix}
\end{equation*}
with
\begin{equation}\label{eq:pij}
    \begin{aligned}
    p_{12}&=\frac{\bar k_p}{2k_i},&
    p_{13}&=\frac{k_d\beta_0}{4k_i},&
    p_{22}&=\frac{k_p\bar k_p}{2k_i^2},\\
    p_{23}&=\frac{\bar k_p k_d\beta_0}{4k_i^2},&
    p_{33}&=\frac{\bar k_p}{2k_i^2},&
    \bar k_p&=k_p-L_g.
    \end{aligned}
\end{equation}
From \eqref{eq:alphabeta}, we have $g(q) - g(q_d) = \alpha(q)(q - q_d)$.
By the scalar consequences of Assumption~\ref{ass:robot}, we have $|g(q) - g(q_d)| \le L_g|q - q_d| = L_g|z_1|$.
Integrating from $q_d$ to $q$ gives
\begin{equation}\label{eq:U_bound}
    |U(q)|
    \le
    \int_{q_d}^{q} L_g|s - q_d|\,ds
    \leq
    \frac{L_g}{2}(q - q_d)^2
    =
    \frac{L_g}{2} z_1^2.
\end{equation}
Hence, $V(z)\geq z^\top(P-\text{diag}\{0,p_{33}L_g,0\})z$. 

We next show that $P-\text{diag}\{0,p_{33}L_g,0\}$ is positive definite.
The leading principal minors of 
\begin{align*}
P-\text{diag}\{0,p_{33}L_g,0\}
&=\begin{bmatrix}
        1&p_{12}&p_{13}\\
        p_{12}&p_{22}-p_{33}L_g&p_{23}\\
        p_{13}&p_{23}&p_{33}
    \end{bmatrix}
\end{align*}
are
\begin{align*}
    \Delta_1&=1>0,\\
    \Delta_2&=\det \begin{bmatrix}
        1&p_{12}\\
        p_{12}&p_{22}-p_{33}L_g
    \end{bmatrix}\\
    &=p_{22}-p_{33}L_g-p_{12}^2
    =\frac{\bar k_p^2}{4k_i^2}>0
\end{align*}
and 
\begin{align*}
        &\det{(P-\text{diag}\{0,p_{33}L_g,0\})}\\
        =&p_{22}p_{33}+2p_{12}p_{13}p_{23}
        +p_{13}^2p_{33}L_g\\
        &-p_{33}^2L_g
        -p_{12}^2p_{33}-p_{13}^2p_{22}\\
    =&\frac{p_{33}\bar k_p}{16 k_i^2}(4 \bar k_p-k_d^2 \beta_0^2).
\end{align*}
From \eqref{eq:gain_condition}, we have
$\bar k_p > \frac{k_d^2}{2m_0}>\frac{k_d^2}{4m_0}>\frac{k_d^2}{4m_1}=\frac{k_d^2\beta_0^2}{4}$.
Therefore, $\det{(P-\text{diag}\{0,p_{33}L_g,0\})}>0$.
Recall that $V(z)\geq z^\top(P-\text{diag}\{0,p_{33}L_g,0\})z$, we know $V$ is a positive definite and radially unbounded function.
Moreover, \eqref{eq:U_bound} gives
\[
    V(z)\leq z^\top(P+\text{diag}\{0,p_{33}L_g,0\})z,
\]
hence $V$ is also bounded above by a positive definite quadratic form.

\noindent\emph{Step 3: Derivative of $V$.}

By \eqref{eq:U_def}, we have $\dot \Psi(q) = \bigl(g(q) - g(q_d)\bigr)\dot q$.
Using $g(q) - g(q_d) = -\alpha(q) z_1$ and $\dot q = -\beta(q) z_2$, we obtain $\dot \Psi(q) = \alpha(q)\beta(q) z_1 z_2$. Consequently, 
\begin{align*}
    \dot V&= \beta(q) z^\top (PA+A^\top P)z
    + 2 \al(q) \beta(q) p_{33}z_1z_2\\
    &=-\beta(q) z^\top R(q) z
\end{align*}
where $R(q)$ is given in \eqref{eq:Rq_full}, and the symbol \(*\) denotes the entries determined by symmetry, i.e.,
\(R_{ij}(q)=R_{ji}(q)\) for \(i>j\). Since $p_{12} = k_i p_{33}$ and $p_{22}=k_p p_{33}$, this matrix reduces to \eqref{dotV}.
Here, we emphasize that the nonlinear functions $\al(q)$ and $\beta(q)$ in \eqref{dotV} satisfy that $|\alpha(q)| \le L_g,~\beta_0 \le \beta(q) \le \beta_1$, for all $ q\in \rr.$ 

By Lemma~\ref{lem:R-positive-1d} below, there exists a constant $c>0$ such that
$\lambda_{\min}(R(\alpha,\beta))> c,~
    \forall \alpha\in[-L_g,L_g],\ \forall \beta\in[\beta_0,\beta_1]$.
Using \eqref{eq:alpha_beta_bounds}, we obtain
$\dot V\le-\beta_0 c \|z\|^2<0,~\forall z \neq 0.$
Since $V$ is bounded above and below by positive definite quadratic forms, standard Lyapunov theory implies that the equilibrium $z=0$ is globally exponentially stable.

Finally, by \eqref{eq:zdef}, we know $z_1 = q_d-q,~z_2 = -\sqrt{m(q)}\,\dot q$.
Since $m(q)$ is uniformly bounded away from zero and infinity, $z_1(t)\to 0$ and $z_2(t)\to 0$ imply
\[
    q(t)\to q_d,
    \qquad
    \dot q(t)\to 0
\]
exponentially.
This completes the proof of Theorem 1.
\hfill$\square$

\begin{lemma}\label{lem:R-positive-1d}
Let $p_{12},~p_{13},~p_{22},~p_{23},~p_{33} $ be defined in \eqref{eq:pij}, and let $R(\al,\beta)$ be given by \eqref{eq:R_lemma}, where $\al \in [-L_g,L_g]$ and $\beta \in [\beta_0,\beta_1]=[\frac{1}{\sqrt{m_1}} ,\frac{1}{\sqrt{m_0}}] $.
If $k_i,~ k_p,~k_d$ satisfy \eqref{eq:gain_condition}, then there exists a positive constant $c$, independent of $\al$ and $\beta$, such that $\lambda_{\min}(R)\geq c$, for all $\al \in [-L_g,L_g]$ and $\beta \in [\beta_0,\beta_1]$.
\end{lemma}
\begin{proof}
Write $R=\begin{bmatrix}
    2k_i p_{13}&v^\top\\
    v&Q
\end{bmatrix}$ with 
\begin{align*}
  v^\top&=\begin{bmatrix}
    (\al+k_p)p_{13}+k_ip_{23}-\frac{1}{\beta}&k_d\beta p_{13}
 \end{bmatrix},\\
 Q&=\begin{bmatrix}
    2(\al+k_p)p_{23}-\frac{2p_{12}}{\beta}&
    \begin{gathered}k_d\beta p_{23}\\{}-\frac{p_{13}}{\beta}\end{gathered}\\
    *&2k_d \beta p_{33}-2p_{23}
\end{bmatrix}.
\end{align*}
Note $k_ip_{13}>0$, it is enough to prove that the Schur complement $S=Q-\frac{1}{2k_ip_{13}}vv^\top$ is positive definite.
A direct computation gives
\begin{align*}
    s_{11}
    &=2(\al+k_p)p_{23}-\frac{2p_{12}}{\beta}\\
    &\quad-\frac{1}{2k_ip_{13}}
    \left[(\al+k_p)p_{13}+k_ip_{23}-\frac{1}{\beta}\right]^2\\
    &=\frac{k_p+\al}{\beta k_i}
    -\frac{2}{\beta_0 \beta^2 k_d}
    - \frac{(\al + L_g)^2 k_d \beta_0}{8k_i^2}\\
    &\geq
    \frac{\bar k_p}{\beta_1 k_i}
    -\frac{2}{\beta_0\beta_1^2 k_d}
    -\frac{L_g^2 k_d \beta_0}{2 k_i^2}
    \\
    s_{22}
    &=2k_d \beta p_{33}-2p_{23}
    -\frac{1}{2k_ip_{13}}(k_d\beta p_{13})^2\\
    &=\frac{k_d}{8k_i^2}
    [\bar k_p(8\beta - 4\beta_0)-k_d^2 \beta_0 \beta^2],\\
    s_{12}
    &=k_d\beta p_{23} -\frac{p_{13}}{\beta}\\
    &\quad-\frac{1}{2k_ip_{13}}
    \left[(\al+k_p)p_{13}+k_ip_{23}-\frac{1}{\beta}\right]
    k_d\beta p_{13}\\
    &=\frac{k_d}{2k_i}\left(1-\frac{\beta_0}{2\beta}\right)
    -\frac{(\al + L_g)k_d^2 \beta_0 \beta}{8 k_i^2}.
\end{align*}
Since $\bar k_p > 4m_1\frac{k_i}{k_d}+\frac{L_g^2}{m_0}\frac{k_d}{k_i}=\frac{4k_i}{\beta_0^2k_d}+\frac{L_g^2k_d\beta_1^2}{k_i}$, we have $$\frac{\bar k_p}{\beta_1k_i}>\frac{4}{k_d\beta_0^2\beta_1}+\frac{k_dL_g^2\beta_1}{k_i^2}\geq \frac{4}{k_d\beta_0\beta_1^2}+\frac{k_dL_g^2\beta_0}{k_i^2}.$$ 
Consequently, we obtain
\begin{align}\label{s11}
    s_{11}>\frac{2}{k_d\beta_0\beta_1^2}+\frac{k_dL_g^2\beta_0}{2k_i^2}>0.
\end{align}
Since $\beta\geq \beta_0$, we have $8\beta-4\beta_0\geq 4\beta$.
By \eqref{eq:gain_condition}, we have $\bar k_p> \frac{k_d^2 \beta_1^2}{2}$. Consequently, 
\begin{align}\label{s22}
    s_{22}\geq \frac{k_d}{8k_i^2}\beta(4\bar k_p - k_d^2\beta_0\beta)\geq \frac{k_d\bar k_p\beta}{4k_i^2}>\frac{ k_d^3 \beta_0 \beta_1^2}{8 k_i^2}>0.
\end{align}
Since $\beta_0\leq \beta$, we have
\[
0<\frac{k_d}{2k_i}\left(1-\frac{\beta_0}{2\beta}\right)
\leq \frac{k_d}{2k_i}.
\]
Since $-L_g\leq \al\leq L_g$ and $\beta\leq \beta_1$, we have
\[
0<\frac{(\al+L)k_d^2\beta_0\beta}{8k_i^2}
<\frac{L_g\beta_0\beta_1k_d^2}{4k_i^2}.
\]
Consequently, we have
\begin{align*}
    |s_{12}|\leq \max\{\frac{k_d}{2k_i},\frac{L_g\beta_0\beta_1 k_d^2}{4k_i^2}\}.
\end{align*}
By \eqref{s11} and \eqref{s22},
\begin{align*}
    s_{11}s_{22}
    &>\left(\frac{2}{\beta_0\beta_1^2 k_d}
    +\frac{L_g^2 k_d \beta_0}{2 k_i^2}\right)
    \frac{ k_d^3 \beta_0 \beta_1^2}{8 k_i^2}\\
    &\quad=
    \frac{k_d^2}{4k_i^2}
    +\frac{ L_g^2\beta_0^2\beta_1^2k_d^2}{16 k_i^4}
    >s_{12}^2.
\end{align*}
Together with $s_{11}>0$, this proves that $S\succ 0$. Since $2k_ip_{13}>0$, the Schur complement argument shows that $R(\alpha,\beta)\succ 0$ for all $\al\in[-L_g,L_g]$ and $\beta\in[\beta_0,\beta_1]$.

Finally, $R(\alpha,\beta)$ depends continuously on $(\alpha,\beta)$, and the set
\[
    \{(\alpha,\beta): \alpha\in[-L_g,L_g],\ \beta\in[\beta_0,\beta_1]\}
\]
is compact. Therefore,
\[
    c:=\min_{\alpha\in[-L_g,L_g],\,\beta\in[\beta_0,\beta_1]}
    \lambda_{\min}\bigl(R(\alpha,\beta)\bigr)>0.
\]
This proves the lemma.
\end{proof}
\subsection{Proof of Theorem 2}\label{sec:proof-nd}

Fix $n\ge2$. We construct a robot manipulator satisfying
Assumption~\ref{ass:robot}.
Write
\[
q=(q_c,q_r),\qquad q_c\in\mathbb R^2,\quad q_r\in\mathbb R^{n-2},
\]
where the $q_r$-component is omitted when $n=2$. Likewise, for any
$v,w\in\mathbb R^n$, write
\[
v=(v_c,v_r),\qquad w=(w_c,w_r),
\qquad v_c,w_c\in\mathbb R^2 .
\]
\noindent\emph{Step 1: Construction of the counterexample robot manipulator.}

Choose $0<\epsilon<1$ and a nonzero vector
$\gamma\in\mathbb R^2$. Define
\[
\phi(r)=2-\epsilon\cos r,~ r\ge 0,\text{ and }\sigma(r)=
\begin{cases}
	\dfrac{\sin r}{r},& r>0,\\[1mm]
	1,& r=0 .
\end{cases}
\]
Note $\nabla_{q_c}\phi(|q_c|) = \epsilon\sigma(|q_c|)q_c$. We set $b(q_c):=\epsilon\sigma(|q_c|)q_c$.
Define
\[
M(q)=
\begin{bmatrix}
	\phi(|q_c|)I_2&0\\
	0&I_{n-2}
\end{bmatrix}.
\]
The Coriolis mapping is defined by
\[
C(q,v)w=
\begin{bmatrix}
	C_2(q_c,v_c)w_c\\
	0
\end{bmatrix},
\]
where
\[
\begin{aligned}
	C_2(q_c,v_c)w_c
	&=
	\frac12(b(q_c)^\top v_c)w_c
	+\frac12(b(q_c)^\top w_c)v_c\\
	&\quad-\frac12(v_c^\top w_c)b(q_c).
\end{aligned}
\]
Equivalently,
\[
C_2(q_c,v_c)
=
\frac12(b(q_c)^\top v_c)I_2
+
\frac12
\left(
v_cb(q_c)^\top-b(q_c)v_c^\top
\right).
\]
Finally, define
\[
\begin{aligned}
U(q)&=\gamma^\top q_c+\cos|q_c|-1,\\
g(q)&=\nabla U(q)
=
\begin{bmatrix}
	\gamma-\sigma(|q_c|)q_c\\
	0
\end{bmatrix}.
\end{aligned}
\]
The functions $\cos|q_c|$ and $\sigma(|q_c|)q_c$ are smooth functions of
$q_c$. Hence $M,C,g$ are smooth.

\noindent\emph{Step 2: Verification of the Euler--Lagrange structural assumptions.}

Since $0<\epsilon<1$, we have $1\leq 2-\epsilon\leq \phi(r)\leq2+\epsilon\leq3$.
Therefore
\[
I_n\leq M(q)\leq 3I_n,
\]
so $(p1)$ holds with $m_0=1$ and $m_1=3$.

Along any differentiable trajectory $q(t)$, we have
\[
\dot M(q)
=
\begin{bmatrix}
	(b(q_c)^\top\dot q_c)I_2&0\\
	0&0
\end{bmatrix}.
\]
Moreover, $2C_2(q_c,\dot q_c)
=
(b(q_c)^\top\dot q_c)I_2
+
\dot q_cb(q_c)^\top
-
b(q_c)\dot q_c^\top$.
Consequently,
\[
\dot M(q)-2C(q,\dot q)
=
\begin{bmatrix}
	b(q_c)\dot q_c^\top-\dot q_cb(q_c)^\top&0\\
	0&0
\end{bmatrix},
\]
which is skew-symmetric. Hence $(p2)$ holds.

For $(p3)$, note that
$
|b(q_c)|
=
\epsilon|\sigma(|q_c|)|\,|q_c|
=
\epsilon|\sin|q_c||
\leq\epsilon $.
Thus, for every $x_c\in\mathbb R^2$, we have
\[
\begin{aligned}
	|C_2(q_c,v_c)x_c| &\leq
	\frac12|b(q_c)||v_c||x_c|+\frac12|b(q_c)||x_c||v_c|\\
	&\quad+\frac12|b(q_c)||v_c||x_c|
	\leq\frac{3\epsilon}{2}|v_c||x_c|.
\end{aligned}
\]
It follows that $\|C(q,v)\|
\leq
\frac{3\epsilon}{2}|v|$,
and $(p3)$ is verified.

The formula for $C_2(q_c,v_c)w_c$ is symmetric in $v_c$ and $w_c$.
Therefore $C(q,v)w=C(q,w)v$, which proves $(p4)$.

Finally, let $h(q_c):=\sigma(|q_c|)q_c$.
Since 
\[
\begin{aligned}
1-\cos |q_c|
&=\sum_{\ell=1}^{\infty}
\frac{(-1)^{\ell+1}}{(2\ell)!}|q_c|^{2\ell}=\sum_{\ell=1}^{\infty}
\frac{(-1)^{\ell+1}}{(2\ell)!}(q_c^\top q_c)^\ell,
\end{aligned}
\]
the function $1-\cos |q_c|$ is smooth on $\mathbb R^2$. Moreover,
$h(q_c)=\nabla_{q_c}(1-\cos |q_c|)$.

Indeed, for $q_c\neq0$,
\[
\nabla_{q_c}(1-\cos |q_c|)
=
\frac{\sin |q_c|}{|q_c|}q_c,
\]
and the identity extends to $q_c=0$ by smoothness.

Thus
\[
g(q)=
\begin{bmatrix}
	\gamma-h(q_c)\\
	0
\end{bmatrix},
\qquad
\nabla^2U(q)
=
\begin{bmatrix}
	-Dh(q_c)&0\\
	0&0
\end{bmatrix}.
\]
It remains to bound $Dh(q_c)$. For $q_c\neq0$, 
a direct differentiation gives
\[
\begin{aligned}
Dh(q_c)=\frac{\sin |q_c|}{|q_c|}I_2+\frac{|q_c|\cos |q_c|-\sin |q_c|}{|q_c|^3}q_cq_c^\top .
\end{aligned}
\]
Equivalently, for every $y\in\mathbb R^2$,
\[
\begin{aligned}
Dh(q_c)y=\frac{\sin |q_c|}{|q_c|}y+\frac{|q_c|\cos |q_c|-\sin |q_c|}{|q_c|^3}q_c(q_c^\top y).
\end{aligned}
\]
We now compute the eigenvalues of $Dh(q_c)$. If $y\perp q_c$, then
$q_c^\top y=0$, and hence $Dh(q_c)y=\frac{\sin |q_c|}{|q_c|}y$.
Thus the tangential eigenvalue is
\[
\lambda_{\mathrm{tan}}(|q_c|)=\frac{\sin |q_c|}{|q_c|}.
\]
On the other hand, in the radial direction $y=q_c$, we have
\[
\begin{aligned}
Dh(q_c)q_c
&=\frac{\sin |q_c|}{|q_c|}q_c+\frac{|q_c|\cos |q_c|-\sin |q_c|}{|q_c|^3}q_c(q_c^\top q_c)\\
&=\frac{\sin |q_c|}{|q_c|}q_c
+\frac{|q_c|\cos |q_c|-\sin |q_c|}{|q_c|}q_c\\
&=\cos |q_c|\, q_c .
\end{aligned}
\]
Thus the radial eigenvalue is
\[
\lambda_{\mathrm{rad}}(|q_c|)=\cos |q_c|.
\]
Since $Dh(q_c)$ is symmetric, for $q_c\neq0$,
\[
\begin{aligned}
\|Dh(q_c)\|
&=\max\left\{
|\lambda_{\mathrm{rad}}(|q_c|)|,
|\lambda_{\mathrm{tan}}(|q_c|)|
\right\}
\\
&=\max\left\{
|\cos |q_c||,
\left|\frac{\sin |q_c|}{|q_c|}\right|
\right\}
\leq1 .
\end{aligned}
\]
By continuity, the same bound holds at $q_c=0$, where $Dh(0)=I_2$.
Consequently,
$$
\|\nabla^2U(q)\|\le1,\qquad q\in\mathbb R^n.$$
Thus $(p5)$ holds with $L_g=1$.



\noindent\emph{Step 3: Closed-loop equation and global existence of solutions.}

We now fix an arbitrary scalar PID gain triple
$(k_p,k_i,k_d)\in\mathbb R^3$ and take the desired setpoint to be
$q_d=0$. Then $e=-q$. We use
\[
    v:=\dot e=-\dot q,
\]
so that $q=-e$ in all occurrences of $M(q),C(q,\cdot)$, and $g(q)$.
By the PID law \eqref{pid}, we have
\[
    \tau=k_i\xi+k_pe+k_dv .
\]
Since the constructed Coriolis mapping is linear in its second argument,
$C(q,-v)(-v)=C(q,v)v$. Therefore the closed-loop system is
\begin{equation}\label{eq:closed-loop-n}
	\left\{
	\begin{aligned}
		\dot\xi&=e,\\
		\dot e &= v,\\
		\dot v
		&=
		M^{-1}(q)\bigl[C(q,v)v+g(q)-k_i\xi-k_pe-k_dv\bigr],
	\end{aligned}
	\right.
\end{equation}
where $q=-e$.
The vector field in \eqref{eq:closed-loop-n} is smooth and globally defined.
Moreover, every solution is forward complete. Indeed, define
\[
\mathcal E(e,v,\xi)
=
\frac12 v^\top M(-e)v+\frac12|e|^2+\frac12|\xi|^2 .
\]
Using the skew-symmetry identity $(p2)$, along solutions of
\eqref{eq:closed-loop-n} we have
\[
\frac{\mathrm{d}}{\mathrm{d}t}M(q(t))-2C(q(t),\dot q(t))
\quad \text{is skew-symmetric}.
\]
Here $\frac{\mathrm{d}}{\mathrm{d}t}M(q(t))$ denotes the time derivative of $M(q(t))$ along
the physical trajectory $q(t)$. Since the present coordinates use
$\dot q(t)=-v(t)$, and since the constructed Coriolis mapping is linear in
its second argument, this identity is equivalently
\[
\frac{\mathrm{d}}{\mathrm{d}t}M(q(t))+2C(q(t),v(t))
\quad \text{is skew-symmetric}.
\]
Hence
\[
\frac12v(t)^\top\frac{\mathrm{d}}{\mathrm{d}t}M(q(t))v(t)
+v(t)^\top C(q(t),v(t))v(t)=0,
\]
and therefore
\[
\begin{aligned}
	\dot{\mathcal E}
	&=
	v^\top(g(q)-k_i\xi-k_pe-k_dv)
	+e^\top v+\xi^\top e .
\end{aligned}
\]
For the constructed system, $g(q)$ is bounded and $M(q)$ is uniformly
positive definite and uniformly bounded. Hence, for a constant $c_K>0$
depending only on the fixed gains and on the constructed system,
\[
\dot{\mathcal E}\leq c_K(1+\mathcal E).
\]
Gronwall's inequality implies that $\mathcal E$ remains bounded on every
finite time interval. Thus every maximal solution of \eqref{eq:closed-loop-n}
is defined on $[0,\infty)$.

\noindent\emph{Step 4: Nonexistence of admissible scalar PID gains.}

The goal of this step is to prove that no scalar PID gain
$K=(k_p,k_i,k_d)\in\mathbb{R}^3$ can achieve global regulation of the
constructed robot manipulator system. 
We distinguish three mutually exclusive cases.

\medskip
\noindent
\textbf{Case 1: $k_i k_d>0$.}

Set $\omega=\sqrt{\frac{k_i}{k_d}}>0$ and define
\[
H(r)
=
\frac{k_i}{k_d}
\left(
\phi(r)+\frac r2\phi'(r)
\right)
+
\sigma(r).
\]
Since
\[
\phi(r)+\frac r2\phi'(r)
=
2-\epsilon\cos r+\frac{\epsilon r}{2}\sin r
\]
and \(k_i/k_d>0\), the oscillatory term \(r\sin r\) implies that \(H\) is
unbounded both above and below. Hence, by the continuity of \(H\) and the
intermediate value theorem, for any prescribed \(k_p\in\mathbb R\), there
exists \(\rho>0\) such that
\[
H(\rho)=k_p .
\]
Let
\[
x_c(t)
=
\rho
\begin{bmatrix}
	\cos\omega t\\
	\sin\omega t
\end{bmatrix},
\quad
x(t)=
\begin{bmatrix}
	x_c(t)\\
	0
\end{bmatrix}.
\]
Then
\[
|x_c(t)|=\rho,\quad
\ddot x_c(t)=-\omega^2x_c(t),\quad
x_c(t)^\top\dot x_c(t)=0 .
\]
We claim that
\[
q(t)=-x(t),\quad
v(t)=\dot x(t),\quad
\xi(t)=
\frac{1}{k_i}
\begin{bmatrix}
	\gamma\\
	0
\end{bmatrix}
-
\frac{k_d}{k_i}\dot x(t)
\]
is an exact solution of \eqref{eq:closed-loop-n}. First,
$\dot q(t)=-\dot x(t)=-v(t)$. Also,
\[
\dot \xi(t)
=
-\frac{k_d}{k_i}\ddot x(t)
=
\frac{k_d}{k_i}\omega^2x(t)
=
x(t)
=
-q(t),
\]
so the integral equation is satisfied. Moreover, because $e=x$ and
$\dot e=\dot x=v$ along this trajectory,
\[
\begin{aligned}
	\tau(t)
	&=
	k_i\xi(t)+k_p e(t)+k_d\dot e(t)\\
	&=
	\begin{bmatrix}
		\gamma\\0
	\end{bmatrix}
	-k_d\dot x(t)+k_px(t)+k_d\dot x(t)\\
	&=
	\begin{bmatrix}
		\gamma\\0
	\end{bmatrix}
	+k_px(t).
\end{aligned}
\]
The $q_r$-components of $q,v,\dot v$ vanish. In the first two
coordinates, using $q_c=-x_c$, $v_c=\dot x_c$, and
$x_c^\top\dot x_c=0$, the right-hand side of the third equation in
\eqref{eq:closed-loop-n} is
\[
C(q,v)v+g(q)-k_i\xi-k_pe-k_dv
\!=\!\!
\begin{bmatrix}
	\begin{gathered}
	\sigma(\rho)
	\!+\!\omega^2\frac{\rho}{2}\phi'(\rho)
	\!-\!k_p
	\\ x_c(t)\end{gathered}\\
	0
\end{bmatrix}.
\]
Since $H(\rho)=k_p$, the last expression equals
\[
\begin{bmatrix}
	-\omega^2\phi(\rho)x_c(t)\\
	0
\end{bmatrix}
=
M(q(t))\dot v(t).
\]
Thus the constructed trajectory is a closed-loop solution.
Taking its value at $t=0$ as the initial condition, uniqueness implies that
the corresponding closed-loop solution is precisely this trajectory.

This solution is periodic and satisfies, for all $t\ge0$,
\[
|e(t)|=|x(t)|=\rho>0,
\qquad
|\dot e(t)|=|\dot x(t)|=\rho\omega>0 .
\]
Therefore $e(t)\not\to0$ and $\dot e(t)\not\to0$. Hence no gain triple
with $k_i k_d>0$ is admissible.

\medskip
\noindent
\textbf{Case 2: $k_i=0$.}

The controller becomes
\[
\tau=k_pe+k_dv .
\]
Consider the forward complete solution with
\[
e(0)=0,\qquad v(0)=0,\qquad \xi(0)=0 .
\]
Suppose, toward a contradiction, that this solution satisfies
\[
e(t)\to0,\qquad v(t)\to0 .
\]
Then $\tau(t)\to0$, and
$
C(q(t),v(t))v(t)\to0
$
because
$
|C(q,v)v|\leq \frac{3\epsilon}{2}|v|^2$.

Since \(e(t)\to0\) and \(v(t)\to0\), we have \(q(t)=-e(t)\to0\).
Moreover,
\[
C(q(t),v(t))v(t)\to0,
\qquad
k_pe(t)+k_dv(t)\to0 .
\]
By the continuity of \(M^{-1}(\cdot)\) and \(g(\cdot)\), the third equation
of the closed-loop system yields
\[
\dot v(t)\to M(0)^{-1}g(0).
\]
Since
\[
M(0)=
\begin{bmatrix}
	(2-\epsilon)I_2&0\\
	0&I_{n-2}
\end{bmatrix},
\qquad
g(0)=
\begin{bmatrix}
	\gamma\\
	0
\end{bmatrix},
\]
we obtain
\[
M(0)^{-1}g(0)
=
\begin{bmatrix}
	(2-\epsilon)^{-1}\gamma\\
	0
\end{bmatrix}
=:a\ne0 .
\]
Therefore, $\dot v(t)\to a\ne0$.

For all sufficiently large $t$, this gives
\[
a^\top\dot v(t)\geq \frac{|a|^2}{2}.
\]
Integrating in time shows that $a^\top v(t)$ grows at least linearly,
contradicting $v(t)\to0$. Thus no gain triple with $k_i=0$ is
admissible.

\noindent\textbf{Case 3: $k_i\ne0$ and $k_i k_d\le0$.}

The regulated equilibrium corresponding to $q_d=0$, if regulation were
achieved, would satisfy
\[
e=0,\qquad v=0,\qquad
\xi=\frac1{k_i}g(0)
=
\frac1{k_i}
\begin{bmatrix}
	\gamma\\
	0
\end{bmatrix}.
\]
Introduce the shifted integral state $\eta=\xi-\frac1{k_i}g(0)$. Then \eqref{eq:closed-loop-n}
becomes
\begin{equation}\label{eq:eta-system-n}
\left\{
	\begin{aligned}
		\dot\eta &= e,\\
		\dot e &= v,\\
	\dot v
		&=M^{-1}(q)\Bigl[-k_i\eta-k_pe-k_dv\\
		&\qquad\quad+C(q,v)v+g(q)-g(0)\Bigr],
	\end{aligned}
\right.
\end{equation}
where $q=-e$. The desired regulated equilibrium is now
$(\eta,e,v)=(0,0,0)$.

We compute the linearization of the right-hand side of
\eqref{eq:eta-system-n} at the origin. Let
\[
\begin{aligned}
F_3(\eta,e,v)
&:=M^{-1}(-e)\Bigl[-k_i\eta-k_pe-k_dv\\
&\qquad\quad+C(-e,v)v+g(-e)-g(0)\Bigr].
\end{aligned}
\]
Since the bracketed term vanishes at the origin, the derivative of
$M^{-1}(-e)$ does not contribute to the linearization. Therefore,
\[
\begin{aligned}
&DF_3(0)(\delta\eta,\delta e,\delta v)\\
=&M^{-1}(0)D\Bigl[-k_i\eta-k_pe-k_dv\\
&\qquad\quad+C(-e,v)v+g(-e)-g(0)\Bigr]_{0}(\delta\eta,\delta e,\delta v).
\end{aligned}
\]
The term $C(-e,v)v$ is $O(|v|^2)$, and hence has zero derivative at the
origin. Moreover, since
\[
g(q)=
\begin{bmatrix}
\gamma-h(q_c)\\0
\end{bmatrix},
\qquad
Dh(0)=I_2,
\]
we have
\[
D[g(-e)-g(0)]_0\delta e
=
\begin{bmatrix}
\delta e_c\\0
\end{bmatrix}.
\]
Thus
\[
\begin{aligned}
&DF_3(0)(\delta\eta,\delta e,\delta v)\\
=&M^{-1}(0)\left[
-k_i\delta\eta-k_p\delta e-k_d\delta v+
\begin{bmatrix}
\delta e_c\\0
\end{bmatrix}
\right].
\end{aligned}
\]
Consequently, the linearized system is
\begin{align*}
    \left\{
    \begin{aligned}
\dot\eta &= e,\\
\dot e &= v,\\
\dot v
&=
M^{-1}(0)\left[-k_i\eta-k_pe-k_dv+
\begin{bmatrix}
e_c&0
\end{bmatrix}^\top\right].
\end{aligned}
    \right.
\end{align*}
Note
\[
M^{-1}(0)
=
\begin{bmatrix}
\frac{1}{2-\epsilon}I_2&0\\
0&I_{n-2}
\end{bmatrix},
\]
the first two coordinate directions $j=1,2$ satisfy
\[
\begin{aligned}
\dot\eta_j&=e_j,\qquad \dot e_j=v_j,\\
\dot v_j&=-\frac{k_i}{2-\epsilon}\eta_j
-\frac{k_p-1}{2-\epsilon}e_j
-\frac{k_d}{2-\epsilon}v_j.
\end{aligned}
\]
Hence each of the first two coordinate directions contributes the scalar
cubic factor
\[
p(s)
=
(2-\epsilon)s^3+k_ds^2+(k_p-1)s+k_i
\]
to the characteristic polynomial of the full linearization.

Under the present assumption $k_i\ne0$ and $k_i k_d\le0$, the polynomial
$p$ has a root in the open right half-plane. If $k_i<0$, then
\[
p(0)=k_i<0,\qquad p(s)\to+\infty
\quad\text{as }s\to+\infty,
\]
so $p$ has a positive real root. If $k_i>0$, then $k_d\le0$. When
$k_d<0$, the sum of the three roots of $p$ is
\[
-\frac{k_d}{2-\epsilon}>0,
\]
so at least one root has positive real part. When $k_d=0$, the sum of the
roots is zero, while
\[
p(0)=k_i>0,\qquad p(s)\to-\infty
\quad\text{as }s\to-\infty.
\]
Thus $p$ has a negative real root $r<0$, and the sum of the real parts of the remaining two roots is $-r>0$. Hence at least one of the remaining roots has positive real part. In all subcases, the linearization of \eqref{eq:eta-system-n} has an eigenvalue in the open right half-plane. Therefore, the origin of \eqref{eq:eta-system-n} is not locally attractive. Hence there exists an initial condition arbitrarily close to the origin whose forward complete solution does not converge to $(\eta,e,v)=(0,0,0)$.

For such a solution, it cannot be true that $e(t) \to 0$ and $v(t) \to 0$.
Indeed, if $e(t)\to0$ and $v(t)\to0$, then Lemma \ref{lem:regulation-implies-full-state} would imply $\eta(t)\to0$, contradicting the choice of the solution. Therefore
\[
e(t)\not\to0
\quad\text{or}\quad
v(t)\not\to0 .
\]
Thus the regulation objective fails. Hence no gain triple with
\[
k_i\ne0,\qquad k_i k_d\le0
\]
is admissible.

The three cases
\[
k_i k_d>0,\quad
k_i=0,\quad
k_i\ne0,\quad \ k_i k_d\le0
\]
are mutually exclusive and cover all scalar PID gain triples in
$\mathbb R^3$. Therefore the constructed manipulator cannot be globally
regulated by any classical scalar-gain PID controller. This proves the
theorem.
\hfill$\square$

\begin{lemma}
\label{lem:regulation-implies-full-state}
Consider the shifted system \eqref{eq:eta-system-n} associated with the constructed manipulator, and suppose that $k_i\neq0$. Let $(\eta(t),e(t),v(t))$ be any forward complete solution of \eqref{eq:eta-system-n}. If
\[
e(t)\to0,\quad v(t)\to0,\quad \text{as } t\to\infty,
\]
then $\eta(t)\to0$ as $t\to\infty.$
\end{lemma}

\begin{proof}
	Let $(\eta(t),e(t),v(t))$ be a forward complete solution such that
	\[
	e(t)\to0,\qquad v(t)\to0 .
	\]
	Throughout the proof, $q(t)=-e(t)$. Hence $q(t)\to0$.
	From \eqref{eq:eta-system-n},
	\[
	k_i\eta
	=
	-M(q)\dot v-k_pe-k_dv+C(q,v)v+g(q)-g(0).
	\]
	Define
	\[
	R(t)
	=
	-k_pe(t)-k_dv(t)+C(q(t),v(t))v(t)+g(q(t))-g(0).
	\]
	By $(p3)$,
	\[
	|C(q(t),v(t))v(t)|
	\leq
	\|C(q(t),v(t))\|\,|v(t)|
	\leq
	L_c|v(t)|^2 .
	\]
	Together with $q(t)\to0$, $v(t)\to0$, and the continuity of $g$, this gives
	\[
	R(t)\to0 .
	\]
	Thus
	\[
	k_i\eta(t)=-M(q(t))\dot v(t)+R(t).
	\]
	
	Fix $T>0$. Integrating over $[t,t+T]$, we obtain
	\[
	\begin{aligned}
	k_i\!\int_t^{t+T}\!\eta(s)\,\mathrm{d}s=\!\int_t^{t+T}\! R(s)\,\mathrm{d}s -\int_t^{t+T}\! M(q(s))\dot v(s)\,\mathrm{d}s.
	\end{aligned}
	\]
	Since $R(t)\to0$ and $T$ is fixed,
	\[
	\int_t^{t+T}R(s)\,\mathrm{d}s\to0 .
	\]
	
	We first show that
	\[
	\int_t^{t+T}M(q(s))\dot v(s)\,\mathrm{d}s\to0 .
	\]
	By integration by parts,
	\[
	\begin{aligned}
	&\int_t^{t+T}M(q(s))\dot v(s)\,\mathrm{d}s\\
    =&M(q(t+T))v(t+T)-M(q(t))v(t)\\
    &-\int_t^{t+T}\frac{\mathrm{d}}{\mathrm{d}s}M(q(s))v(s)\,\mathrm{d}s .
	\end{aligned}
	\]
	The endpoint terms tend to zero because $M(q)$ is uniformly bounded and
	$v(t)\to0$.
	
	For the remaining integral, recall that $(p2)$ concerns the physical
	velocity $\dot q(t)$ and the time derivative of $M(q(t))$ along the
	physical trajectory:
	\[
	\frac{\mathrm{d}}{\mathrm{d}t}M(q(t))-2C(q(t),\dot q(t))
	\quad \text{is skew-symmetric}.
	\]
	In the present coordinates $\dot q(t)=-v(t)$. Since the constructed
	Coriolis mapping is linear in its second argument,
	$C(q(t),-v(t))=-C(q(t),v(t))$. Therefore the preceding identity becomes
	\[
	\frac{\mathrm{d}}{\mathrm{d}t}M(q(t))+2C(q(t),v(t))
	\]
	skew-symmetric along solutions. Therefore, for every unit vector $x$,
	\[
	x^\top \frac{\mathrm{d}}{\mathrm{d}t}M(q(t))x
	=
	-2x^\top C(q(t),v(t))x .
	\]
	Because $M(q)$ is symmetric, $\frac{\mathrm{d}}{\mathrm{d}t}M(q(t))$ is symmetric. Hence
	\[
	\begin{aligned}
	\left\|\frac{\mathrm{d}}{\mathrm{d}t}M(q(t))\right\|
	&=\sup_{|x|=1}\left|x^\top \frac{\mathrm{d}}{\mathrm{d}t}M(q(t))x\right|\\
	&\leq2\|C(q(t),v(t))\|
	\leq2L_c|v(t)|.
	\end{aligned}
	\]
	Therefore, $\left|\frac{\mathrm{d}}{\mathrm{d}t}M(q(t))v(t)\right|
	\leq
	2L_c|v(t)|^2$.
	
	Since $v(t)\to0$, it follows that, for fixed $T>0$,
	\[
	\int_t^{t+T}\frac{\mathrm{d}}{\mathrm{d}s}M(q(s))v(s)\,\mathrm{d}s\to0 .
	\]
	Consequently,
	$
	\int_t^{t+T}M(q(s))\dot v(s)\,\mathrm{d}s\to0.
	$
	Since $k_i\neq0$,
	$
	\int_t^{t+T}\eta(s)\,\mathrm{d}s\to0 .
	$
	
	Finally, since $\dot\eta=e$ and $e(t)\to0$,
	\[
	\begin{aligned}
		\left|
		\eta(t)-\frac1T\int_t^{t+T}\eta(s)\,\mathrm{d}s
		\right|
		&\leq
		\frac1T\int_t^{t+T}|\eta(t)-\eta(s)|\,\mathrm{d}s \\
		&\leq
		\frac1T\int_t^{t+T}\int_t^s|e(\ell)|\,\mathrm{d}\ell\,\mathrm{d}s \\
		&\leq
		T\sup_{\ell\in[t,t+T]}|e(\ell)|
		\to0 .
	\end{aligned}
	\]
	Since
	\[
	\frac1T\int_t^{t+T}\eta(s)\,\mathrm{d}s\to0,
	\]
	we conclude that
	\[
	\eta(t)\to0 ,
	\]and the proof of Lemma 2 is complete.
\end{proof}

\textbf{Remark 4 }
	Lemma~\ref{lem:regulation-implies-full-state} is not a consequence of the
	integrator chain
	\[
	\dot\eta=e,\qquad \dot e=v
	\]
	alone. Its proof uses the particular Euler--Lagrange form of the third
	equation in \eqref{eq:eta-system-n}. More precisely, the identity
	\[
	\frac{\mathrm{d}}{\mathrm{d}t}M(q(t))-2C(q(t),\dot q(t))
	\quad \text{is skew-symmetric}
	\]
	is applied with the physical velocity $\dot q(t)=-v(t)$, and therefore
	becomes
	\[
	\frac{\mathrm{d}}{\mathrm{d}t}M(q(t))+2C(q(t),v(t))
	\quad \text{is skew-symmetric}.
	\]
	Together with the growth bound on $C(q,v)$, this allows one to show that,
	whenever $v(t)\to0$,
	\[
	\int_t^{t+T} M(q(s))\dot v(s)\,\mathrm{d}s\to0
	\]
	for each fixed $T>0$. This is the key step that forces the moving average
	of $\eta(t)$ to converge to zero, and then $\dot\eta=e\to0$ converts this
	average convergence into pointwise convergence.
	
	Such a conclusion is false for a general third-order system of the form
	\[
	\dot\eta=e,\quad
	\dot e=v,\quad
	\dot v=F(\eta,e,v).
	\]
	For example, consider the scalar smooth system
	\[
	\dot\eta=e,\quad
	\dot e=v,\quad
	\dot v=-e-2v .
	\]
	Then $e$ satisfies
	\[
	\ddot e+2\dot e+e=0,
	\]
	and hence
	\[
	e(t)\to0,\qquad v(t)\to0
	\]
	for every initial condition. However, if
	\[
	\eta(0)=1,\quad e(0)=0,\quad v(0)=0,
	\]
	then
	\[
	e(t)\equiv0,\quad v(t)\equiv0,\quad \eta(t)\equiv1.
	\]
	Thus $e(t)\to0$ and $v(t)\to0$ do not imply $\eta(t)\to0$ for a general
	third-order system. The implication in Lemma~\ref{lem:regulation-implies-full-state}
	therefore reflects the mechanical structure of \eqref{eq:eta-system-n}, rather
	than a generic property of cascaded integrator systems.
   
\section{Conclusion}\label{sec:conclusion}
To address a long-standing open problem  on global regulatability of the classical PID control for robot manipulators, we provide a fairly complete and degree-of-freedom-dependent solution for PID control with triple parameters $(k_p,k_i,k_d)$ in this paper. 
To be specific, in one degree of freedom, an explicit set of PID gain inequalities guarantees global exponential regulation for every initial condition and setpoint. In dimension $n\ge2$, the same class of assumptions does not imply global regulatability by scalar PID gains: a smooth mechanical counterexample was constructed for which every scalar gain triple fails for some initial condition and setpoint. These results indicate that multi-degree-of-freedom robot manipulator regulation by PID requires additional structure, richer gain matrices, model-dependent compensation, or modified nonlinear integral actions beyond the scalar-gain architecture considered here.


\enlargethispage{2\baselineskip}

\end{document}